%% file: root.tex
\documentclass[letterpaper, 10 pt, conference]{ieeeconf}  % Comment this line out if you need a4paper

\IEEEoverridecommandlockouts                              % This command is only needed if 
\usepackage{graphics} % for pdf, bitmapped graphics files
\usepackage{epsfig} % for postscript graphics files
\usepackage{times} % assumes new font selection scheme installed
\usepackage{amsmath} % assumes amsmath package installed
\usepackage{amssymb}  % assumes amsmath package installed
\usepackage{xcolor}
\usepackage{cancel}
\usepackage{pifont}
\usepackage[normalem]{ulem}
\usepackage{bbold}
\graphicspath{{Images/}}

\renewcommand{\vec}[1]{\boldsymbol{\mathbf{#1}}}

\renewcommand{\theta}{\vartheta}
\DeclareMathOperator*{\minimize}{minimize}

\newcommand{\black}{\color{black}}

\usepackage{nomencl}
\makenomenclature

\newtheorem{theorem}{Theorem}[section]
\newtheorem{lemma}{Lemma}[section]

\usepackage[compress]{cite}

\title{\LARGE \bf
Real-Time Generation of Near-Minimum-Energy Trajectories via Constraint-Informed Residual Learning 
}
\author{Domenico Dona'$^{1,*}$,  Giovanni Franzese $^{2}$, Cosimo Della Santina $^{2,3}$, Paolo Boscariol$^{4}$ and Basilio Lenzo$^{1}$% <-this % stops a space
\thanks{This research was funded by Fondazione Aldo Gini (Call 2023). This research was conducted at TU Delft and the corresponding author acknowledges the Erasmus+ program for the financial support.}% <-this % stops a space
\thanks{$^{*}$Corresponding author,  {\tt\small domenico.dona@phd.unipd.it}}
\thanks{$^{1}$Department of Industrial Engineering,
        University of Padua, 35131 Padova, Italy}%
\thanks{$^{2}$Department of Cognitive Robotics, ME,
Delft University of Technology, 2628 CD Delft, The Netherlands}
\thanks{$^{3}$Institute of Robotics and Mechatronics, German Aerospace Center (DLR), 82234 Oberpfaffenhofen, Germany}
\thanks{$^{4}$Department of Management and Engineering (DTG),
        University of Padua, 36100 Vicenza, Italy}
}

\begin{document}

\maketitle
\thispagestyle{empty}
\pagestyle{empty}

%%%%%%%%%%%%%%%%%%%%%%%%%%%%%%%%%%%%%%%%%%%%%%%%%%%%%%%%%%%%%%%%%%%%%%%%%%%%%%%%
\input{abstract}
\input{introduction_2}
\input{related_works_2}

\input{method}
\input{results}
\input{conclusion}

\bibliographystyle{IEEEtran}  
\bibliography{main}

%%%%%%%%%%%%%%%%%%%%%%%%%%%%%%%%%%%%%%%%%%%%%%%%%%%%%%%%%%%%%%%%%%%%%%%%%%%%%%%%

%%%%%%%%%%%%%%%%%%%%%%%%%%%%%%%%%%%%%%%%%%%%%%%%%%%%%%%%%%%%%%%%%%%%%%%%%%%%%%%%
\section*{APPENDIX}\label{apdx}
This section demonstrates that a GP with a kernel in the form of \eqref{eq:customkernel} respects propriety \eqref{eq:therule}.
First of all, it is useful to prove the following Lemma.
\begin{lemma}\label{lemma:1}
    The derivatives of the symmetric function \eqref{eq:customkernel} $k_{10}$ and $k_{11}$ are both zero for $\xi = 0$ and $\xi = 1$.
\end{lemma}
\begin{proof}
    Given the symmetric function $k(\vec{x}_1, \vec{x}_2)$, where $\vec{x}_i = (\xi_i, \vec{q}_0, \vec{q}_f, t_f)$, defined as:
    \begin{equation}
        k(\vec{x}_1, \vec{x}_2) = s(\xi_1) s(\xi_2) k_\textup{RBF}(\vec{x}_1, \vec{x}_2)
    \end{equation}
    where $k_\textup{RBF} = \exp \left( -\frac{\lVert \vec{x}_1 - \vec{x}_2 \rVert_2^2}{\ell^2} \right)$ is the RBF kernel, and $s(\cdot)$ is the \textit{scaling function} defined in \eqref{eq:scaling}. Given the following notation for the derivatives of the symmetric function:
    \begin{equation}\label{eq:k10}
        \begin{gathered}
            k_{10}(\vec{x}_1, \vec{x}_2) = \frac{\partial k(\vec{x}_1, \vec{x}_2)}{\partial \xi_1} \\
            k_{11}(\vec{x}_1, \vec{x}_2) = \frac{\partial^2 k(\vec{x}_1, \vec{x}_2)}{\partial \xi_1 \partial \xi_2}
        \end{gathered}
    \end{equation}
    The expression for $k_{10}$ is:
    \begin{equation}
        k_{10} = s'(\xi_1) s(\xi_2) k_\textup{RBF} + s(\xi_1) s(\xi_2) \frac{\partial}{\partial \xi_1} k_\textup{RBF} 
    \end{equation}
    and for $k_{11}$ is:
    \begin{multline}\label{eq:k11}
        k_{11} =  s'(\xi_1) s'(\xi_2) k_\textup{RBF} +  s'(\xi_1) s(\xi_2) \frac{\partial}{\partial \xi_2} k_\textup{RBF} + \\
        + s(\xi_1) s'(\xi_2) \frac{\partial}{\partial \xi_1} k_\textup{RBF} + s(\xi_1) s(\xi_2) \frac{\partial^2}{\partial \xi_1 \partial \xi_2} k_\textup{RBF} 
    \end{multline}
    where $s'(\xi) = 4 \xi^3 - 6 \xi^2 + 2 \xi$ which equals zero for both $\xi = 0$ and $\xi = 1$. Given that $s(0)=s(1)=s'(0)=s'(1)=0$, from \eqref{eq:k10} and \eqref{eq:k11}, it follows that $k_{10}$ and $k_{11}$ are both zero.
\end{proof}

\noindent Now it is possible to prove the following Theorem.
\begin{theorem}
    A symmetric function of the type \eqref{eq:customkernel} generates a family of functions that satisfy \eqref{eq:therule}.
\end{theorem}
\begin{proof}
    We consider some training points consistent with the BCs, $\{\vec{X}, \vec{y} \}$, as well as the vector of test points $\hat{\vec{x}} = [\hat{\vec{x}}_0, \hat{\vec{x}}_1]$, where $\hat{\vec{x}}_i = [i,\, \mathcal{X}], \; i = 0,1$. The GP mean $\vec{\mu}$ at the two test points $\hat{\vec{x}}$ can be computed as:
    \begin{equation}
        \vec{\mu} = \vec{K}(\hat{\vec{x}},\vec{X}) (\vec{K}(\vec{X},\, \vec{X})+ \mathbb{1}_n \sigma_n^2)^{-1} \vec{y}
    \end{equation}
    We will show that $\vec{\mu} = \vec{0}$ by showing $\vec{K}(\vec{X},\, \hat{\vec{x}}) = \mathbb{0}$, i.e. the \textit{cross-correlation} between the training points and the test points is zero. From \eqref{eq:customkernel}, it is evident that the kernel is zero $\forall \vec{X}$ with respect to the test points $\hat{\vec{x}}$. The same holds for the covariance matrix $\vec{\Sigma}$ by using its definition:
    \begin{equation}
      \scalebox{0.9}{$
        \vec{\Sigma} = \vec{K} (\vec{\hat{\vec{x}}},\hat{\vec{x}}) - \vec{K}(\hat{\vec{x}},\vec{X}) (\vec{K}(\vec{X},\, \vec{X})+ \mathbb{1}_n \sigma_n^2)^{-1}\vec{K}(\hat{\vec{x}},\vec{X}) $}
    \end{equation}
   Using the fact that $\vec{K}(\hat{\vec{x}},\vec{X}) = \mathbb{0}$ and $\vec{K} (\vec{\hat{\vec{x}}},\hat{\vec{x}}) = \mathbb{0}$, also $\vec{\Sigma}$ is zero. This shows that the samples converge to the mean at the boundaries, since:
    \begin{equation}
        \vec{y}_\textup{s} = \vec{L} \vec{z} + \vec{\mu}
    \end{equation}
    where $\vec{y}_\textup{s}$ denotes the samples at the boundaries, $\vec{\mu}$ is the mean at the boundaries, and $\vec{L}$ is the lower triangular matrix from the Cholesky decomposition of $\vec{\Sigma}$, and $\vec{z} \sim \mathcal{N}(\vec{0}, \mathbb{1})$ is a Gaussian distributed random vector. Since $\vec{\Sigma} = \mathbf{0}$, $\vec{L} = \mathbf{0}$.

    For derivative information, it is sufficient to use Lemma~\ref{lemma:1} and observe that $\vec{K}_{11}(\hat{\vec{x}}, \hat{\vec{x}})$ and $\vec{K}_{10}(\hat{\vec{x}}, \vec{X})$ are both zero for $\xi = 0$ and $\xi = 1$; both the mean and the covariance of the derivative are zero:
    \begin{equation}
    \scalebox{0.8}{$
        \begin{gathered}
            \vec{\mu}' = \vec{K}_{10}(\hat{\vec{x}}, \vec{X}) (\vec{K}(\vec{X},\, \vec{X})+ \mathbb{1}_n \sigma_n^2)^{-1} \vec{y}, \\
            \vec{\Sigma}' = \vec{K}_{11}(\hat{\vec{x}}, \hat{\vec{x}}) - \vec{K}_{10}(\hat{\vec{x}}, \vec{X}) (\vec{K}(\vec{X},\, \vec{X})+ \mathbb{1}_n \sigma_n^2)^{-1} \vec{K}_{01}(\vec{X}, \hat{\vec{x}}).
        \end{gathered}$}
    \end{equation}
    
\end{proof}

% -------------------- ACKNOWLEDGMENT -------------------- %
% \section*{ACKNOWLEDGMENT}

\addtolength{\textheight}{-12cm}   % This command serves to balance the column lengths

%%%%%%%%%%%%%%%%%%%%%%%%%%%%%%%%%%%%%%%%%%%%%%%%%%%%%%%%%%%%%%%%%%%%%%%%%%%%%%%%

\end{document}

%% file: abstract.tex
\begin{abstract}
Industrial robotics demands significant energy to operate, making energy-reduction methodologies increasingly important. Strategies for planning minimum-energy trajectories typically involve solving nonlinear optimal control problems (OCPs), which rarely cope with real-time requirements. In this paper, we propose a paradigm for generating near minimum-energy trajectories for manipulators by learning from optimal solutions. Our paradigm leverages a residual learning approach, which embeds boundary conditions while focusing on learning only the adjustments needed to steer a standard solution to an optimal one. 
Compared to a computationally expensive OCP-based planner, our paradigm achieves 87.3\% of the performance near the training dataset and 50.8\% far from the dataset, while being two to three orders of magnitude faster.
\end{abstract}

%% file: introduction_2.tex
\section{Introduction}

The emergence of ``green'' policies in industry is crucial for mitigating climate change. In the context of robotics, enhancing the efficiency of robotic systems can significantly reduce energy consumption and operational costs \cite{carabin2017review}.

Trajectory Optimization (TO) has emerged as a cost-effective approach for reducing energy consumption of robotic systems, such as single-axis systems \cite{lenzo2024real}, quadrotors \cite{morbidi2021practical}, and manipulators \cite{field1996iterative}. 
TO is used to plan trajectories that minimize total energy expenditure by solving the associated Optimal Control Problem (OCP). However, closed-form solutions for OCPs are rarely available, requiring the use of numerical routines. Consequently, these numerical routines limit the Real-Time (RT) deployment of TO, confining its applicability to precomputed (off-line) trajectories, thus posing a significant limitation in adaptive scenarios.

In this paper, we address this limitation by introducing a surrogate model paradigm. Trained on precomputed optimal trajectories, the surrogate model inherently enables RT deployment.
\black
Instead of learning the whole solution, we learn only the quantity (or \textit{residual}) that steers a standard (or prior) solution toward the optimal one, leveraging the idea of \textit{residual learning}, as illustrated in Fig.~\ref{fig:abstract}. The key idea is that the solution of the planning problem for a  fixed-time-point-to-point motion is rapidly available, e.g., cubic law; what remains is the (residual) component required for achieving optimality.\\
By proper regressor design, the proposed paradigm (i) allows the embedding of Boundary Conditions (BCs) as hard constraints, (ii) requires less training data, and (iii) allows deciding the resolution of the solution. 
Moreover, by using probabilistic regressors, the performance is enhanced by selecting the solution that best satisfies the optimality requirements from multiple samples; the uncertainty measure is used in an Active Learning (AL) data aggregation routine. The idea is depicted in Fig.~\ref{fig:abstract2}. 
\begin{figure}[t!]
  \centering
\includegraphics[width=\linewidth]{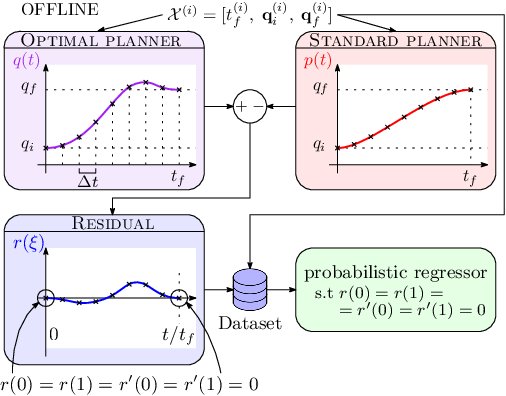}
  \caption{Schematic of the proposed \textit{residual} learning paradigm. Offline, a dataset of residuals is generated and used to train a specifically designed probabilistic regressor that embeds the boundary conditions.
  }
  \label{fig:abstract}
\end{figure}

In summary, this work proposes a \textbf{real-time} capable planner for (near) minimum-energy fixed-time-point-to-point planning problems. This is achievedusing a \textbf{residual learning} paradigm that allows the embedding of BCs as \textbf{hard constraints} and reduces the number of required training data. 
The proposed paradigm is tested using two probabilistic regression frameworks, namely Neural Network (NN) ensembles and Gaussian Processes.\\
% Active Learning (AL) capabilities are evaluated for both \textcolor{blue}{types of} regressors
The proposed paradigm is validated through simulated experiments on three robotic systems, namely a pendulum, a SCARA robot, and a six-axis manipulator

\begin{figure}[t!]
  \centering
  \includegraphics[width=\linewidth]{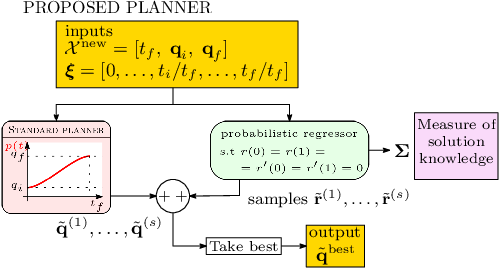}
  \caption{ The proposed planner: the trained probabilistic regressor is deployed online by summing its output with the standard planner output. Moreover, the regressor outputs the epistemic uncertainty of the solution, that can be used as a measure of the lack of knowledge (data).
  }
  \label{fig:abstract2}
\end{figure}

%% file: related_works_2.tex
\section{Related Works}

Two major families of numerical methods for solving OCP can be distinguished: (i) direct and (ii) indirect methods \cite{bryson2018applied}. 
The former transforms the $\infty$-dimensional optimization problem into a Nonlinear Programming problem (NLP). Alternatively, the seconds formulate the problem as a Two-Point Boundary Value Problem (TPBVP) resulting from the Pontryagin Maximum Principle (PMP).

A main idea to speed up the generation process is to avoid the \textit{numerical routines} and to use surrogate models.
An early example can be seen in \cite{bemporad2000explicit}, where precomputed solutions were deployed online for a Model Predictive Control (MPC) application.
However, a clear structure for the solution is not always available. Recently, data-driven methods have emerged as a viable option for speeding up the TO problem \cite{hauser2016learning, sambharya2023end, cheng2018real, zang2022machine, tang2019data, sanchez2018real, zhu2019deep, tailor2019learning, ferede2024end, tang2018learning, banerjee2020learning, wu2024deep, berniker2015deep, Bency2019, hochreiter1997long, vaswani2017attention, guffanti2024transformers, celestini2024transformer, briden2023constraint}. 
% Regressors are trained using optimal trajectories to learn mappings from the inputs of a problem to convenient outputs.
The proposed regression models are trained to generate optimal trajectories, but different input-output choices are proposed in the literature. As such, we can identify four main categories: (i) indirect models, (ii) state feedback models, (iii) parameterized solution models.
% ------------------ indirect models ------------------ %
Indirect models leverage the PMP ordinary differential equations to retrieve the solution. Given the initial and final state, the regressor predicts the initial co-state and the total time (if not provided as input).
% The regressor is used to learn the input - initial co-states/total time map. 
Relevant examples are found in \cite{cheng2018real,tang2019data,zang2022machine} for spacecraft applications. However, to retrieve the trajectory, the TPBVP has to be solved, resulting in possible convergence issues given the quality of the predicted initial co-state \cite{bryson2018applied}.
% Anyway, these examples still have to solve the TPBVP.

% ------------------ state feedback ------------------ %
State feedback models are based on the fact that the OCP solution is a function of the state. In \cite{sanchez2018real,zhu2019deep}, the state-feedback map was learned for the optimal landing problem using a Neural Network (NN), while quadcopter applications can be found in \cite{tailor2019learning, ferede2024end}. Different recurrent models can be found in the literature, based on LSTM architectures \cite{Bency2019} or Transformer models \cite{guffanti2024transformers, celestini2024transformer, briden2023constraint}. Note that, these examples do not guarantee BC satisfaction in the hard sense, resulting in possible unfeasible solutions.

% ------------------ parameterized solutions ------------------ %
Finally, parameterized modules output a complete sequence of the states and/or control inputs given the initial/final state and the total time. The resulting trajectory can also be used as the initial guess to warm up an optimizer \cite{tang2018learning, banerjee2020learning}.
For reducing the dimensionality of the output, an autoencoder can be used to learn a latent representation of the trajectory \cite{berniker2015deep}. 
However, these approaches suffer from the fact that either the length and/or the resolution of the trajectory has to be fixed to perform the regression.

In light of the above, learning solutions from data is promising for RT applications but existing methods: (i) do not impose BCs as hard constraints, (ii) do not leverage prior knowledge, (iii) still need numerical routines, or, (iv) have a predetermined length/resolution for the output. The next section introduces our proposed residual learning paradigm.

%% file: method.tex
\section{METHOD}\label{sec:method}
\noindent The paradigm can be divided into the following steps:
\begin{enumerate}
    \item Formulate the minimum-energy problem as an OCP and solve it numerically to create a dataset of residuals.
    \item Define a structure for the regressors that (i) embeds the BCs for the residuals and (ii) gives uncertainty information of the output.
    \item When a new set of problem variables is given, generate a near-optimal solution using the selected regressor and (eventually) evaluate the uncertainty.
    \item If needed, upgrade the training set by generating new data based on uncertainty information, i.e., active learning.
\end{enumerate}

\subsection{Optimal Control Problem formulation}
We are interested in fixed-time-point-to-point (FT-PTP) trajectories for manipulators, as usually task time is decided externally by production constraints. The study case can be formalized as a fixed-time fixed-endpoint OCP. 
The functional that we want to minimize is the energy expenditure, i.e. the integral over time of the electrical power. We take as state the vector of configurations and velocities $\vec{x} = (\vec{q},\; \vec{v}) \in \mathbb{R}^{2n}$ and as control the torque exerted by the motors $\vec{u} = \vec{\tau} \in \mathbb{R}^{n}$, where $n$ is the number of degrees-of-freedom of the system. Using the DC equivalent circuit model, the electrical power takes the following quadratic form:
\begin{equation}\label{eq:cost}
    \mathcal{P} = \sum_i^n \mathcal{P}_i = \sum_i^n \left( r_i u_i^2 + v_i u_i \right) = \vec{u}^\intercal \vec{R} \vec{u} + \vec{u}^\intercal \vec{N} \vec{x} 
\end{equation}
where first term is due to Joule losses and the second is due to the mechanical power. The control weight matrix $\vec{R}$ is a diagonal matrix whose \textit{i-th} diagonal term is:
\begin{equation}
    r_i = \frac{R_a^i}{{k_t^i}^2}
\end{equation}
where $R_a^i$ is the armature resistance and $k_t^i$ the torque constant of the \textit{i-th} motor, respectively.
The mixed state-control term matrix $\vec{N}$ only selects the velocities from the state vector:
\begin{equation}
    \vec{N} = \begin{bmatrix}
        \mathbb{0}_n \\
        \mathbb{1}_n
    \end{bmatrix}
\end{equation}
where $\mathbb{0}_n$ and $\mathbb{1}_n$ are the $n \times n$ zero and identity matrices, respectively. With the convenient definitions above, it is possible to state the desired problem as following: 
\begin{subequations}\label{eq:problem}
\begin{align}
    \minimize_{\vec{x}(\cdot),\vec{u}(\cdot)} &\quad \mathcal{E} = \int_0^{t_f} \mathcal{P}(\vec{x}(t),\vec{u}(t)) \, \mathrm{d}t \label{eq:problem1a} \\
    \text{subject to:} &\quad  \vec{x}(0) = \vec{x}_0, \quad \vec{x}(t_f) = \vec{x}_f \label{eq:problem1b} \\
                &\quad \dot{\vec{x}} = \vec{f}(\vec{x},\vec{u},t) \label{eq:problem1c}
\end{align}
\end{subequations}
where $t_f$ is the prescribed total time, $\vec{x}_0 = (\vec{q}_0,\; \vec{v}_0)$ and $\vec{x}_f = (\vec{q}_f,\; \vec{v}_f)$ are the BCs. Since it is a PTP motion the velocities at the boundaries are zero ($\vec{v}_0 = \vec{0}$ and $\vec{v}_f = \vec{0}$).
The dynamics $\vec{f}$ can be explicited by splitting the position and velocity components as follows, 
\begin{equation}\label{eq:dynamics}
\dot{\vec{x}} = 
\begin{cases}
    \dot{\vec{q}} = \vec{v}  \\
    \dot{\vec{v}} = \texttt{FD}(\vec{q}, \vec{v}, \vec{u}) 
\end{cases}
\end{equation}
where \texttt{FD($\cdot$)} is the \texttt{F}orward \texttt{D}ynamics map. 

\subsection{Regressor design}
The \textit{explicit} solution of the OCP can be seen as a map from the inputs of the problem, i.e. the time $t$ and the trajectory parameters $\mathcal{X}$ to the trajectory $\vec{x}(t)$. In particular, $t \in [0,\; t_f]$ and $\mathcal{X} = (t_f,\;\vec{q}_0,\;\vec{q}_f)$.

In this work,  we parameterize the solution of the optimal control problem (OCP) in the \textit{explicit} manner where we focus solely on predicting the position over time, $\vec{q}(t)$, i.e., 
\begin{equation}
    \vec{q}(t) = \vec{g} (t, \mathcal{X}).
\end{equation}
This choice alleviates the curse of output dimensionality of the regressor; nevertheless, the velocity can still be retrieved numerically or analytically from the position.

Using a dataset of optimal trajectories, one can learn an approximation $\tilde{\vec{g}}$ for the mapping $\vec{g}$. However, this ``naive'' approach has the following problems\footnote{An ablation study to demonstrate the statement is reported in Sect.~\ref{sssec:ablation}}:
\begin{enumerate}
    \item The trajectories must be learned from scratch.
    \item There are no guarantees of BCs satisfaction.
    \item It requires a high amount of data.
\end{enumerate}
The above problems can be addressed using the \textit{residual learning} paradigm presented in this paper. A schematic of the paradigm is depicted in Fig.~\ref{fig:abstract}. 
\subsubsection{Residual Learning of explicit solutions}
The key idea is that we know how to plan a FT-PTP trajectory, but we do not know how to plan it optimally. Therefore, we want to learn only the \textit{residual} component that steers our standard (prior) solution towards the optimal one. This reduces the burden of the regressor by learning only a part of the solution.  Mathematically, let us define the \textit{prior} as $\vec{p}(t)$ and the \textit{residual} as $\vec{r}(t)$:
\begin{equation}
    \vec{q}(t) = \vec{p}(t) + \vec{r}(t)
\end{equation}
where the prior $\vec{p}(t)$ can be, for instance, a cubic polynomial that satisfies the BCs in \eqref{eq:problem1b}. 

\subsubsection{Boundary Conditions Satisfaction}

Since the prior satisfies the BCs, the residual must have zero values at the boundaries, i.e., 
\begin{equation}\label{eq:res}
    \vec{r}(0) = \vec{r}(t_f) =  \dot{\vec{r}}(0) = \dot{\vec{r}}(t_f) =\vec{0}.
\end{equation}
At this point, even if the values at the boundaries are the same, their position depends on the value of $t_f$. For overcoming this issue,  we can define the normalized time $\xi$ as the ratio between the physical time $t$ and the total time $t_f$, such that the boundaries are at $\xi = 0$ and $\xi = 1$. The mapping $\vec{r}(\cdot)$ can be approximated using a regressor $\tilde{\vec{r}}(\xi, \mathcal{X})$ that, according to \eqref{eq:res}, must satisfy the following:

\begin{equation}\label{eq:therule}
    \tilde{\vec{r}} \mid \tilde{\vec{r}}(0) = \tilde{\vec{r}}(1)  = \tilde{\vec{r}}'(0) = \tilde{\vec{r}}'(1) = \vec{0}
\end{equation}
where $\square'$ denotes the partial derivative with respect to $\xi$. The second argument has been omitted for brevity. Depending on the type of nonlinear regressor chosen, different techniques can be applied to satisfy \eqref{eq:therule}. In the following, NNs and GPs are considered, as they are two of the most established nonlinear regressors.

% \subsection{Neural Networks}
To enforce property \eqref{eq:therule} using a NN, it is sufficient to multiply the output of the NN element-wise by a ``scaling term'' $s$, as shown in Fig.~\ref{fig:nn}. One possible choice for $s$ is:
\begin{equation}\label{eq:scaling}
    s(\xi) = \xi^2 (1 - \xi)^2
\end{equation}

\begin{figure}[t!]
  \centering
  \includegraphics[width=\linewidth]{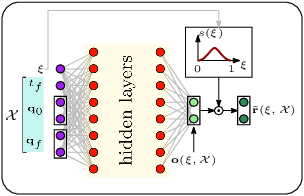}
  \caption{Proposed neural network scheme. The scaling function is applied element-wise.}
  \label{fig:nn}
\end{figure}

and calling $\vec{o}$ the output of the NN, then, the regressor:
\begin{equation}\label{eq:NNmod}
    \tilde{\vec{r}}(\xi, \mathcal{X}) = s(\xi) \odot \vec{o}(\xi, \mathcal{X})
\end{equation}
satisfies \eqref{eq:therule}. The symbol $\odot$ is used for element-wise multiplication.

Recalling the objectives outlined at the beginning of Section~\ref{sec:method}, at this stage, we lack uncertainty information for the output. This issue can be addressed by using an ensemble of neural networks, providing two main benefits: (i) estimation of the epistemic uncertainty using the standard deviation of outputs, and (ii) selection of the best solution (in energy sense) among the set of predictions. The ensemble approach introduces only a minor computational cost, as it can be easily parallelized.

% \subsection{Gaussian Processes}
Another common nonlinear regression model is given by GPs. Similar to NNs, GPs do not satisfy property \eqref{eq:therule} out-of-the-box. However, although the process is less straightforward, one can modify the GP kernel to enforce it.
To do so, we propose taking a general-purpose kernel $k_g(\cdot, \cdot)$, such as a Radial Basis Function (RBF), and multiplying it by two ``scaling'' functions $s(\cdot)$ that depends only on the normalized time $\xi$. This custom kernel is defined as:
\begin{equation}\label{eq:customkernel}
    k(\xi_1,\xi_2) = s(\xi_1) k_g(\xi_1,\xi_2) s(\xi_2); 
\end{equation}
with $s(\xi)$ being the scaling function introduced in \eqref{eq:scaling}. Here, $\xi_1$ and $\xi_2$ are the normalized time points at which the function is evaluated. The kernel \eqref{eq:customkernel} satisfies \eqref{eq:therule}, as proven in the Appendix. Unlike NNs, where an ensemble was needed to estimate uncertainty, this regressor inherently provides epistemic uncertainty information. Additionally, multiple solutions are sampled from the distribution, and the one with the lowest energy consumption is selected.

\subsubsection{Active Learning}\label{ssec:active} 
Our paradigm reduces data quantity by enforcing structure through the prior. However, solution quality may vary across the input space due to limited data, as data generation is costly. 
% Since data are costly, it is not feasible to cover the entire input space. 
However, we can query the optimal planner as needed to update the dataset. We propose employing an AL strategy based on the idea that the further the input is from the existing dataset, the higher the uncertainty in the output. As a result, the epistemic uncertainty of the output serves as a measure of the quality of the solution, allowing us to generate new data where the uncertainty is higher. The idea is depicted in Fig.~\ref{fig:active}.
\begin{figure}[t]
  \centering
  \includegraphics[width=0.9\linewidth]{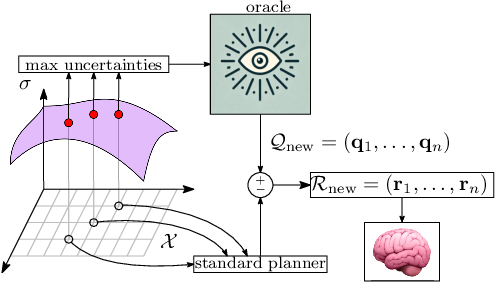}
  \caption{Active learning approach: the uncertainty of the solution is used to generate new data.}
  \label{fig:active}
\end{figure}

The way new data is incorporated differs significantly between the types of regressors used. GPs can easily incorporate new data by updating their training dataset. On the other hand, NNs do not possess the same flexibility when it comes to incorporating new information. 
% Various strategies have been suggested in the literature to address this issue \cite{kirkpatrick2017overcoming, li2017learning}. 
One common strategy for NNs is fine-tuning, which involves training the model for additional epochs while enriching the original batches with new data.

%% file: results.tex
\section{RESULTS}\label{sec:result}

\subsection{Datasets}
The models designed in Section~\ref{sec:method} are compared across three different mechanical systems, namely an actuated pendulum, a SCARA robot, and a 6-axis manipulator. The datasets of residuals are generated using the Rockit~\cite{gillis2020effortless} OC framework and the IPOPT~\cite{biegler2009large} NLP solver. The dynamics of the pendulum system is modeled using Casadi \cite{Andersson2019}, while the Pinocchio library~\cite{carpentier:hal-03271811} is employed for the robotic manipulators.
The standard (or prior) planner is defined as a cubic polynomial law.

The SCARA dataset was constructed by noting the system's symmetry with respect to the first initial joint, which made it possible to reduce the dataset dimension.

Finally, the UR5e dataset was constructed in a slightly different way. The idea was to stress the proposed paradigm while leveraging prior knowledge of the workspace. In particular, no simplification was made for the first joint symmetry. At the same time, the initial and final configurations were chosen by discretizing the workspace in a grid and then retrieving the joint values by means of Inverse Kinematics (IK). The orientation of the initial and final pose was chosen equal in all the trajectories as well as the configuration along the trajectory.\footnote{The inverse kinematics generally has $8$ solutions for the UR5e away from singularities. We use the same righty - below - noflip.} The details on the number of trajectories, data, and sampling frequency are summarized in Tab.~\ref{tab:dataset}. We will refer to these data as \textit{inside dataset}, while data outside these boundaries will be referred to as the \textit{outside dataset}.

\begin{table}[ht!]
\caption{Details of the training datasets' construction. For the UR5e, the $\vec{q}_j$ quantities are defined as\textsuperscript{*} $(x,\;y,\; z)$ positions of the end effector.}
\label{tab:dataset}
\begin{center}
\begin{tabular}{l c c c }
\hline
 & Pendulum & SCARA & UR5e \\
\hline
\# data & $808$ & $43344$ & $33658$ \\
\# trajectories & $8$ & $144$ & $231$ \\
sampling freq. (Hz) & $100$ & $100$ & $50$ \\
ranges $t_f$ (s) & $[1.0, 1.5]$ & $[2.5, 3.5]$ & $[2.5, 3.0]$ \\
\# samples $t_f$ & $2$ & $3$ & $2$ \\
ranges $\vec{q}_i$ (rad or m) & $[-\pi/4 , \pi/4]$ & $[0,0]$ & $[0.3, 0.5]$ \\
 &   & $\left[ 0,  \frac{3}{4}\pi \right]$ & $ 0.0 \times -0.5$ \\
 &   & & $[0.0, 0.5]$ \\
\# samples $\vec{q}_i$ & $4$ & $1 \times 4$ & $3 \times 3 \times 3$ \\
ranges $\vec{q}_f$ (rad or m) & $[-\pi/4, \pi/4]$ & $[0, \pi ]$ & $[0.3, 0.5]$ \\
 &   & $[0, \frac{3}{4}\pi]$ & $[0.0, 0.5]$ \\
 &   &   & $[0.0, 0.5]$ \\
\# samples $\vec{q}_f$ & $4$ & $4 \times 4$ & $3 \times 3 \times 3$ \\
 add. requirement & $\lvert q_f \rvert > \lvert q_i\rvert$ & $q_{f,2} > q_{i,2}$ & $\sum_i \lvert\Delta x_i\rvert > 0.1$\\
  & $ \lvert \Delta q \rvert > \pi/16$  &   & $z_f - z_i>0$ \\
\hline
\end{tabular}
\end{center}
\footnotesize\textsuperscript{*} with abuse of notation
\end{table}

The dynamic parameters of the first two systems are in the \texttt{supplementary} material. For the 6-axis robot, the dynamical parameters were taken from the company's official \texttt{urdf}, while the friction parameters were taken from \cite{clochiatti2024electro}, averaging forward and inverse parameters.

\subsection{Regressors settings}

The hyperparameters for the three test cases are summarized in Table~\ref{tab:hyperNN} for the NNs; in all cases, AdamW  is used as optimizer \cite{loshchilov2017decoupled} and the activation functions are $\tanh$. The ensemble has $6$~models for the UR5e case and $10$~for the other two. For the GPs, we used a Stochastic Variational Gaussian Process (SVGP) due to the large amount of data; the Adam optimizer was used with a learning rate of $10^{-2}$ for $100$ epochs, and $100$~inducing points were selected. The number of samples was $100$ for the pendulum and $10$ for the other two cases. During training, we split the data into $80\%$ for training and $20\%$ for testing. Timing results are obtained using a laptop equipped with an AMD R9 8945HS CPU, 32~GB of RAM, and an NVIDIA RTX 4060 GPU. The NNs are developed using PyTorch~\cite{paszke2019pytorch}, while the GPs are implemented using GPyTorch~\cite{gardner2018gpytorch}.

\begin{table}[ht!]
\caption{Common hyperparameters of the NNs. LR is for learning rate, WD for weight decay, HL for hidden layer, AF for activation function.}
\label{tab:hyperNN}
\begin{center}
\begin{tabular}{l c c c c }
\hline
Test case & LR & WD & \# HLs & HL width\\
\hline
Pendulum &$10^{-3}$ & $10^{-2}$ & $2$ & $50$ \\
SCARA &$10^{-3}$ & $10^{-3}$ & $2$ & $50$ \\
UR5e &$10^{-3}$ & $10^{-1}$ & $2$ & $100$ \\
\hline
\end{tabular}
\end{center}
\end{table}

\subsection{Discussion}
\subsubsection{Ablation study}\label{sssec:ablation}
First of all, to evaluate the superiority of the residual learning paradigm, an ablation study was performed. This study aims to demonstrate that without a proper prior and accurate regressor design, results become unfeasible. Specifically, Fig.~\ref{fig:abl_ex} illustrates a trajectory example trained with a vanilla NN and a vanilla GP, showing boundary violations \textbf{on training data}. Statistical evidence is provided in Fig.~\ref{fig:abl_stats}, where the two models were evaluated across the entire dataset. On average, the violations inside the dataset, considering only the position, amount to about $0.2$~rad for the NN and $0.3$~rad for the GP, suggesting a lack of feasibility.

\begin{figure}[t!]
  \centering
  \includegraphics[width=\linewidth]{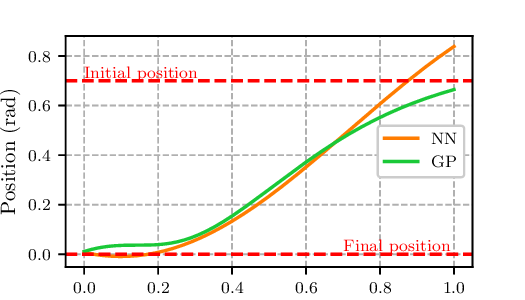}
  \caption{Example trajectories of the NN and GP naive models. The violations of the boundary conditions are evident.}
  \label{fig:abl_ex}
\end{figure}

\begin{figure}[t!]
  \centering
  \includegraphics[width=\linewidth]{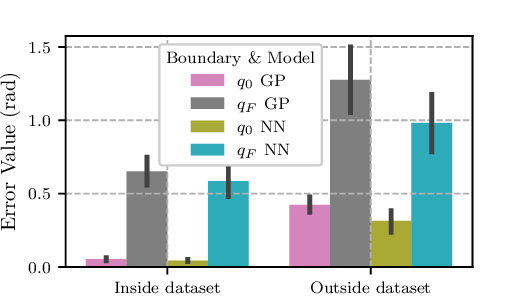}
  \caption{Violations of the boundaries inside and outside the dataset for the NN and GP naive models.}
  \label{fig:abl_stats}
\end{figure}

\subsubsection{Sampling effect}
The energy savings for the pendulum case compared to the standard solution are shown in Fig.~\ref{fig:pendulum}. The figure displays the best sample from both the NN ensemble and GP sampling, along with the mean results $\mu_\textrm{NN}$ and $\mu_\textrm{GP}$. It is worth noting that sampling improves model performance. A summary is reported in Tab.~\ref{tab:savingsamples}.

\begin{table}[ht!]
\caption{Savings with respect to the standard solution inside and outside the dataset for the pendulum.}
\label{tab:savingsamples}
\begin{center}
\begin{tabular}{c c c c c c }
\hline
Dataset & OCP & NN & $\mu_\textrm{NN}$ & GP & $\mu_\textrm{GP}$ \\
\hline
Inside (\%)&$28.5$ & $25.8$ & $25.2$ & $21.8$ & $18.5$ \\
Outside (\%)&$31.6$ & $6.8$ & $5.0$ & $2.8$ & $0.4$\\ 
\hline
\end{tabular}
\end{center}
\end{table}

\begin{figure}[t!]
  \centering
  \includegraphics[width=\linewidth]{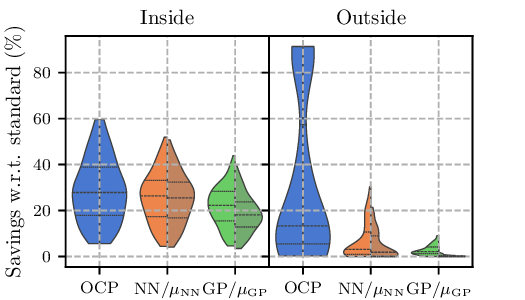}
  \caption{Inside and Outside savings with respect to the standard solution for the pendulum. The right split of the violin represent the output mean of the regressors, showing the advantage of sampling.}
  \label{fig:pendulum}
\end{figure}

\subsubsection{Energy saving capabilities}\label{sssec:energysaving}

Remarkably, both the NN- and GP-based models enable energy savings inside and outside the dataset. A comparison across the robot test cases is depicted in Fig.~\ref{fig:energyoverall}, alongside the detailed results for the pendulum in Fig.~\ref{fig:pendulum}. Overall, the NN-based model performs slightly better than the GP-based model. Further discussion on the comparison of the two models is provided in Section~\ref{sssec:active}.

\begin{figure}[t!]
  \centering
  \includegraphics[width=\linewidth]{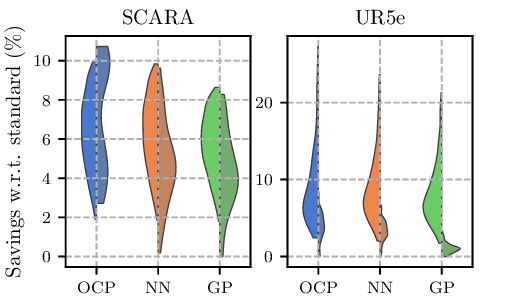}
  \caption{Energy savings with respect to the standard solution for the manipulators. Brighter data points correspond to inside dataset savings, while darker ones represent outside dataset savings.}
  \label{fig:energyoverall}
\end{figure}

\subsubsection{Real-time capability}\label{sssec:realtime}
The times required to compute the solution are reported in Fig~\ref{fig:timingoverall}, where the real-time limit is set as a tenth of the average task time $t_f$. Notably, both models guarantee real-time capabilities for the UR5e case.

\begin{figure}[t!]
  \centering
  \includegraphics[width=\linewidth]{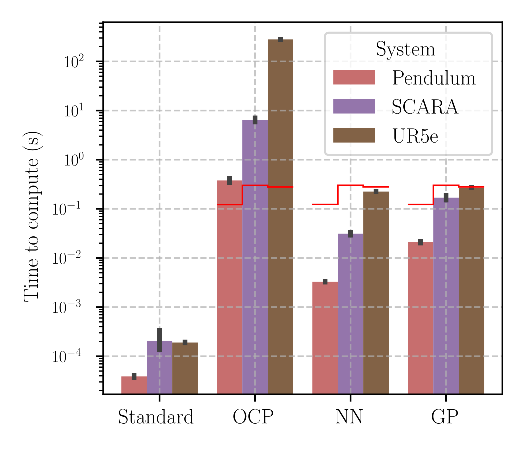}
  \caption{Computation times for each model in seconds. The piecewise constant \textcolor{red}{red} line represents the real-time limit, defined as one-tenth of the total task time, $t_f$.}
  \label{fig:timingoverall}
\end{figure}

\subsubsection{Active Learning}\label{sssec:active}
As discussed in Section~\ref{sssec:energysaving}, the NN-based model slightly outperforms the GP-based one. The main difference between the two models, as introduced in Section~\ref{ssec:active}, is in the way they can incorporate new data. For this reason, the $4$ cases with the higher uncertainty of the outside dataset were used to update the models, for the pendulum system. The savings \textit{before} and \textit{after} the AL phase are reported in Tab.~\ref{tab:activelearning}.

\begin{table}[ht!]
\caption{Savings with respect to the standard solution before and after the active learning phase.}
\label{tab:activelearning}
\begin{center}
\begin{tabular}{l c c c}
\hline
Case & pre-AL (\%) & post-AL (\%) & OCP reference (\%) \\
\hline
GP  &  2.98 & 50.93 & 53.84  \\
NN & 17.63 & 49.65 & 55.46 \\
\hline
\end{tabular}
\end{center}
\end{table}

Notably, the GP can incorporate data more efficiently, and is much easier since the ensemble of NNs requires additional training for each model. In particular, $200$ more epochs with a learning rate of $10^{-4}$ were performed for each model. On the other hand, before AL, the NN-based model performs better. This suggests that depending on the deployment phase, one can choose one over the other.

\subsubsection{Examples}
Simulated examples are provided to illustrate differences between the two regressors. In addition to the regressor solutions, the prior (Cubic) and the OCP solution are shown. Within the dataset, the regressors closely follow the OCP solution, while outside the dataset, greater variability is observed, with the Gaussian Process (GP) solution tending toward the prior. The results are reported for each testcase, both inside and outside the dataset, and for each model; in Fig.\ref{fig:comparison_pend} for the pendulum, and in Fig.\ref{fig:comparison_robot} for the SCARA and UR5e robots.

\begin{figure*}[t]
\centering
\includegraphics[width=\textwidth]{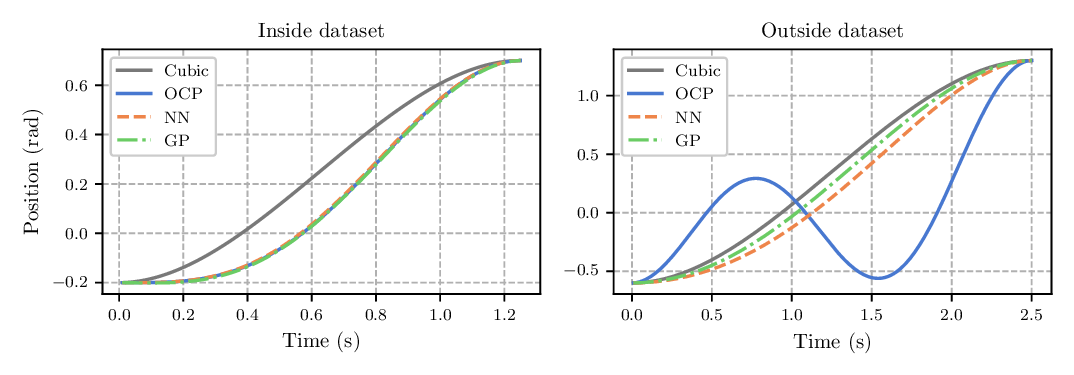}
\caption{Comparison of example trajectories for the pendulum, both within and outside the dataset. Inside the dataset, the regressor accurately captures the solution, while outside the dataset, the complexity of the solution is less represented. Notably, the Gaussian Process (GP) solution tends to align with the prior.}
\label{fig:comparison_pend}
\end{figure*}

\begin{figure*}[t]
\centering
\includegraphics[width=\textwidth]{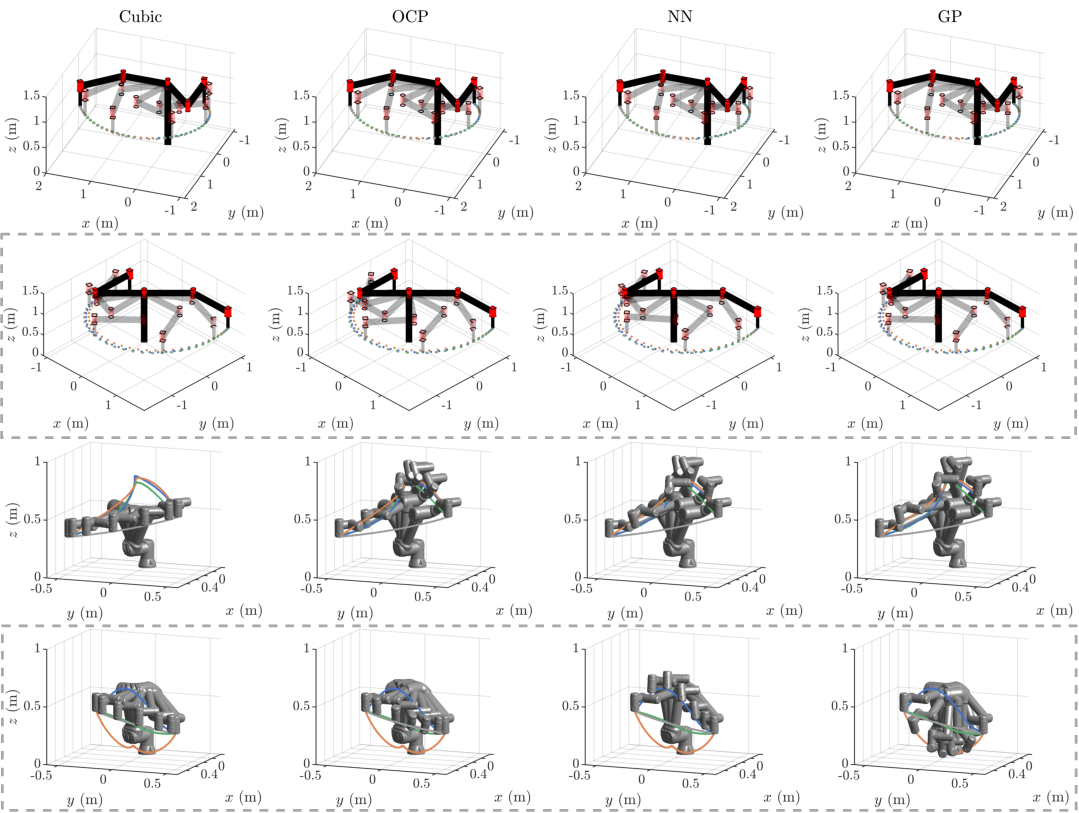}
\caption{Comparison of example trajectories for the robotic manipulators. The grey dashed boxes indicate regions outside the dataset.}
\label{fig:comparison_robot}
\end{figure*}

% \begin{table}[ht!]
% \caption{Parameters used to create the dataset.}
% \label{tab:parameters}
% \begin{center}
% \begin{tabular}{c c c c}
% \hline
% Parameter & unit & Value/range & \# samples\\
% \hline
% $I$ & kg m$^2$ & $1.0$ & / \\
% $f_v$ & N m s/ rad & $1.0\cdot10^{-1}$ & / \\
% $m g \ell$ & N m & $14.7$ & / \\
% $R$ & $\frac{1}{\text{N m s}}$ & 1.0 & / \\ 
% $q_i$ & rad & $-\pi/4 \div \pi/4$ & $6$\\
% $q_f$ & rad & $-\pi/4 \div \pi/4$ & $6$\\
% $t_f$ & s & $1.0\div 1.5$ & $2$\\
% \hline
% \end{tabular}
% \end{center}
% \end{table}

% \begin{table}[ht!]
%     \caption{Parameters of the SCARA robot.}
%     \label{tab:SCARA}
%     \begin{center}
%     \begin{tabular}{c c c c}
%     \hline
%     Parameter & Unit & Value/Range & \# Samples\\
%     \hline
%     $m_1$ & kg & $1.0$ & / \\
%     $m_2$ & kg & $1.0$  & / \\
%     $\ell_1$ & m & $0.5$ & / \\
%     $\ell_2$ & m & $0.5$ & / \\
%     $I_1$ & kg m$^2$ & $1.0\cdot10^{-1}$ & / \\
%     $I_2$ & kg m$^2$ & $1.0\cdot10^{-1}$ & / \\
%     $f_{v,1}$ & N m s/rad & $1.0\cdot10^{-1}$ & / \\
%     $f_{v,2}$ & N m s/rad & $1.0\cdot10^{-1}$ & / \\
%     $R_1$ & $\frac{1}{\text{N m s}}$ & $1.33\cdot10^{-3}$ & / \\
%     $R_2$ & $\frac{1}{\text{N m s}}$ & $1.33\cdot10^{-3}$ & / \\  
%     $q_{1,i}$ & rad & $0.0$ & / \\
%     $q_{2,i}$ & rad & $0.0 \div \frac{3}{4}\pi$ & 4 \\
%     $q_{1,f}$ & rad & $0.0 \div \pi$ & 4 \\
%     $q_{2,f}$ & rad & $0.0 \div \frac{3}{4}\pi$ & 4 \\
%     $t_f$ & s & $2.5 \div 3.5$ & 3 \\
%     \hline
%     \end{tabular}
%     \end{center}
% \end{table}

%% file: conclusion.tex
\section{CONCLUSIONS}\label{sec:conclu}
% The application of trajectory optimization to robots in industrial scenarios has been limited by the inability to generate solutions in real-time. Non-linear regressors have demonstrated the ability to learn from optimal solutions, but they do not guarantee the satisfaction of boundary conditions out-of-the-box. Additionally, we may want to include new data as evidence arises from operations.

In this paper, by focusing only in the residual between an optimal planner solution and a given standard planner, the objective of learning optimal solutions has been achieved. Through numerical simulation, the advantage of the proposed paradigm has been shown to be twofold. Firstly, it possible to embed boundary conditions and secondly reduce the quantity of training data required.

By using variational regressors, it is possible to select the best output from the samples and using uncertainty information to upgrade the training dataset.
% In this letter, we propose an approach based on a residual learning paradigm that has a dual advantage. First, it allows us to learn only part of the solution, specifically the part that steers a standard solution toward an optimal one. Second, it enables the hard embedding of boundary conditions within the regressor. It has been demonstrated that even near the dataset, naive regressors without the inclusion of boundary conditions do not guarantee feasibility. Two types of non-linear regressors were used: neural network ensembles and Gaussian processes.
% Regressors with multiple outputs and uncertainty information have a twofold advantage. First, we can compare the outputs and take the best in terms of energy. Second, the uncertainty can be used as a measure of how far we are from distribution and therefore update the dataset if needed, in an active learning fashion.
% Overall, 
Between the two regressor frameworks used, Gaussian Processes turned out to be more flexible in incorporating new data during the active learning phase, while the ensemble of neural networks demonstrated better results before adding new evidence.

% The proposed work has shown how the use of a prior in trajectory learning allows for fewer data and guarantees feasibility. 
Future work will consider the inclusion of other types of constraints, such as torque or speed limits, to make the paradigm applicable in more demanding situations.